%% file: main.tex
\documentclass{article}

\usepackage{amsmath,amssymb,amsfonts}
\usepackage{algorithmic}
\usepackage{graphicx}
\usepackage{textcomp}
\usepackage{xcolor}
\usepackage{adjustbox}
\usepackage{pgfplots}
\usepgfplotslibrary{groupplots}
\usepackage{hyperref}
\usepackage{tikz}
\usetikzlibrary{positioning}
\usetikzlibrary{shapes.geometric}
\usepackage{hhline}
\usepackage{multirow}
\usepackage{wrapfig}
\usepackage{xcolor}
\usepackage{xspace}
\usepackage{mathtools}
\usepackage{dsfont}
\usepackage{amsthm}
\usepackage{siunitx}
\usepackage{placeins}
\usepackage{enumitem}
\usepackage{ifthen}
\usepackage{cite}
\newboolean{detailed}

\usepgfplotslibrary{colorbrewer}
\usetikzlibrary{patterns}

\usepackage{subcaption}
\usepackage[english]{babel}
\usepackage{blindtext}

\usepackage{hyperref,natbib}

\usepackage{microtype}
\usepackage{graphicx}
\usepackage{booktabs} 

\usepackage{hyperref}

\usepackage[accepted]{icml2020}

\icmltitlerunning{Unbiased Loss Functions for Extreme Classification With Missing Labels}

\setboolean{detailed}{true}

\input{macros}
\begin{document}

\twocolumn[
\icmltitle{Unbiased Loss Functions for Extreme Classification With Missing Labels}




\begin{icmlauthorlist}
\icmlauthor{Erik Schultheis}{to}
\icmlauthor{Mohammadreza Qaraei}{to}
\icmlauthor{Priyanshu Gupta}{goo}
\icmlauthor{Rohit Babbar}{to}
\end{icmlauthorlist}

\icmlaffiliation{to}{Aalto University, Helsinki, Finland}
\icmlaffiliation{goo}{IIT, Kanpur, India}

\icmlcorrespondingauthor{Rohit Babbar}{firstname.lastname@aalto.fi}


\vskip 0.3in
]



\printAffiliationsAndNotice{}  

\begin{abstract}
The goal in extreme multi-label classification (\textbf{XMC}) is to tag an instance with a small subset of relevant labels from an extremely large set of possible labels. 
In addition to the computational burden arising from large number of training instances, features and labels, problems in XMC are faced with two statistical challenges, (i) large number of \lq tail-labels\rq -- those which occur very infrequently, and (ii) missing labels as it is virtually impossible to manually assign \textit{every} relevant label to an instance. 
In this work, we derive an unbiased estimator for general formulation of loss functions which decompose over labels, and then infer the forms for commonly used loss functions such as hinge- and squared-hinge-loss and binary cross-entropy loss.
We show that the derived unbiased estimators, in the form of appropriate weighting factors, can be easily incorporated in state-of-the-art algorithms for extreme classification, thereby scaling to datasets with hundreds of thousand labels.
However, empirically, we find a slightly altered version that gives more relative weight to tail labels to perform even better. We suspect is due to the label imbalance in the dataset, which is not explicitly addressed by our theoretically derived estimator.
Minimizing the proposed loss functions leads to significant improvement over existing methods (up to 20\% in some cases) on benchmark datasets in XMC.
\end{abstract}

\input{intro}
\input{theory}

\input{experiments}

\input{related}
\input{conclusion}

\FloatBarrier
\bibliography{LSHTC-biblio}
\bibliographystyle{icml2020}

\end{document}

%% file: macros.tex
\newcommand{\set}[1]{\left\{#1\right\}}
\newcommand{\real}{\mathds{R}}

\newcommand{\defmap}[3]{#1: #2 \longrightarrow #3}
\DeclareMathOperator{\expect}{\mathds{E}}
\newcommand{\prob}[1]{\mathcal{P}\!\left\{ #1 \right\}}
\newcommand{\ind}{\mathds{1}}

\newcommand{\parabel}{\texttt{Parabel}\xspace}

\newcommand{\ndcg}[1]{\ensuremath{\mathrm{nDCG}@#1}}
\newcommand{\dcg}[1]{\ensuremath{\mathrm{DCG}@#1}}
\newcommand{\psp}[1]{\ensuremath{\mathrm{PSP}@#1}}
\newcommand{\p}[1]{\ensuremath{\mathrm{P}@#1}}

\DeclareMathOperator{\rank}{rank}

\newcommand{\yv}{\mathbf{y}}
\newcommand{\yhatv}{\mathbf{\hat{y}}}

\newcommand{\ndcgatk}{\ensuremath{\ndcg{k}}\xspace}%

\newcommand{\opl}{l^*_{+}}
\newcommand{\oml}{l^*_{-}}
\newcommand{\ppart}[1]{\max\left(#1, 0\right)}
\newcommand{\numlabels}{N}

\newtheorem{theorem}{Theorem}
\newtheorem{corollary}{Corollary}

\newcommand{\detail}[1]{\ifthenelse{\boolean{detailed}}{#1}{}}
\newcommand{\detailcolor}{}
\newcommand{\detailalt}[2]{\ifthenelse{\boolean{detailed}}{#2}{#1}}
\newcommand{\detaildisplay}[1]{
\ifthenelse{\boolean{detailed}}
{\begin{align}
	#1
\end{align}}{$#1$}
}

%% file: intro.tex
\section{Introduction}
Extreme Multi-label Classification (\textbf{XMC}) refers to supervised learning where each training/test instance is labeled with small subset of relevant labels that are chosen from a large set of possible target labels.
Problems consisting of extremely large number of labels are common in various domains such as annotating large encyclopedia \citep{dekel2010multiclass, partalas2015lshtc}, image-classification \citep{deng2010does} and next word prediction \citep{mikolov2013efficient}.
It has also been noticed that the framework of XMC can be effectively leveraged to address learning problems arising in recommendation systems and web-advertising \citep{Agrawal13, prabhu2014fastxml, jain2019slice}.
For the case of recommendation systems, by learning from other similar users' buying patterns, this framework can be used to recommend a small subset of relevant items from a large collection of all possible items. 

With applications in a diverse range, designing effective machine learning algorithms to solve XMC has become a key research challenge.
From the computational aspect of the learning problem, building effective extreme classifiers is faced with a scaling challenge arising due to large number of (i) output labels, (ii) input training instances, and (iii) input features. It is not unlikely to have the dimensionality of the above dataset statistics to be $\mathcal{O}(10^6)$. Two properties of datasets in XMC which pose further statistical challenges, (i) Long-tail distribution of instances among labels, and (ii) Missing labels, are discussed next.

\subsection{Tail Labels} \label{sec:tail}
An important statistical feature of the datasets in XMC is that a large fraction of labels are tail labels, i.e., those which have very few training instances (also referred to as a fat-tailed distribution and Zipf's law). 
Typically the label frequency distribution follows a power law, an example of which is shown in Figure \ref{fig:powerlaw} for a publicly available benchmark WikiLSHTC-325K dataset, consisting of approximately 325,000. 
 
Concretely, let $n_{(r)}$ denote the number of occurrences of the $r$-th ranked label, when ranked in decreasing order of number of training instances that belong to that label, then
\begin{equation}\nonumber
n_{(r)} = n_{(1)}r^{-\beta},
\label{eq:powerlaw1}
\end{equation} 
\begin{figure}[!h]
\scalebox{0.99}{
\centering
\includegraphics[width=0.5\textwidth]{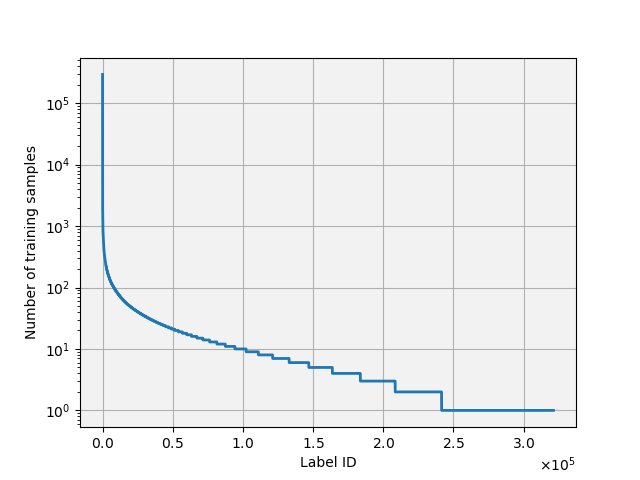}
}
\caption{
Label frequency in WikiLSHTC-325K. X-axis shows the label IDs sorted by the frequency of positive instances and Y-axis gives the actual frequency (on log-scale). 
More than half of the labels have fewer than 10 training instances.
}
\label{fig:powerlaw}
\end{figure}
where $\beta>0$ denotes the exponent of the power law.

Tail labels exhibit diversity of the label space, and contain informative content not captured by the head or torso labels.
Indeed, by predicting well the head labels, yet omitting most of the tail labels, an algorithm can achieve high accuracy \citep{wei2019does}.
Such behavior is not typically desirable in real world applications \citep{babbar2019data}.
In movie recommendation systems, for instance, the head labels correspond to popular blockbusters---most likely, the user has already watched these.
However, the tail of the distribution corresponds to less popular yet equally favored films, like independent movies.
These are the movies that the recommendation engine should ideally focus on.
A similar argument applies to search engine development \citep{radlinski2009redundancy} and hash-tag recommendation in social media \citep{denton2015user}. However, effectively predicting tail-labels can be an enormous challenge due to the extreme data imbalance problem, where a given tail label appears in only a couple of (positive) instances and does not appear in millions of other (negative) training instances.

\subsection{Missing Labels}
In addition to having unfavourable statistics, when learning to classify tail labels it has been shown that one also needs to account for missing labels in the training data.
In a dataset where the labels for each example are chosen from a label space with thousands of elements, it is impossible to explicitly check for the presence or absence of
each label, some examples will have missing labels. Even worse, the chance for a label to be missing is higher for tail labels than for head labels. In the movie example, this
means that there are more people who would have liked an independent movie, but did not because never seeing it, than there are people who would have liked a blockbuster but
never saw it. However, we can typically assume that most people who claim to like a movie actually do so, i.e. that we do not have significant amounts of spurious positive
labels in the training set. This leads to the propensity model introduced in \citep{Jain16}, formally presented in \autoref{sec:theory}.

In \citep{Jain16}, it was shown that it is possible to calculate an unbiased estimate of certain loss functions and evaluation metrics, even if the available data has missing labels.
They further showed that using such an estimate for training results in better performance on tail labels.
However, the analysis left out some important loss functions, such as hinge- and squared-hinge-loss as well as binary cross-entropy. In this paper we derive a way of calculating
unbiased estimates for all loss functions that separate over binary labels. However, it turns out that such estimates may lose important properties of the original loss function,
e.g. an unbiased estimate of the (convex) hinge loss need no longer be convex. Therefore, we provide an alternative derivation, in which we consider the hinge loss as a piecewise-linear
convex upper-bound on the 0-1 loss, and recommend using a piecewise-linear convex upper-bound on the unbiased estimate of the 0-1 loss. 

We show that the derived unbiased estimators, in the form of appropriate weighting factors for loss functions, can be readily incorporated in state-of-the algorithms for extreme classification, and hence easily scale to datasets with hundreds of thousand labels. Empirically, the efficacy of the proposed loss functions is demonstrated by exhibiting superior performance to existing methods, with relative improvements of as much as  20\% compared to some of the recently proposed tree-based methods for extreme classification \citep{Prabhu18b}.


%% file: theory.tex
	\section{Theory}
\label{sec:theory}

In the extreme classification setting, it is not possible for a human annotator to
consider every possible label when deciding which labels to assign to a given data point.
Instead, they will look at an example and assign a set of fitting labels that comes to mind.
It is reasonable to assume that any label assigned in such fashion will be correct, i.e. if
the annotator where asked whether the label belonged to the example, they would confirm this.
The converse is not necessarily true: If one were to ask the annotator for each label that
was not chosen whether it was relevant for the example, it is likely that some would be
considered relevant. 

To capture this effect, the notion of propensity is introduced. The propensity of a label
(for an example) is defined as the probability of the label being present, given that when
explicitly asked the ground-truth annotator would confirm it to be present. An empirical
model for estimating propensities from label frequencies is presented in \cite{Jain16}. 

For any given label $i$, we denote with $Y_i \in \set{0, 1}$ whether the label is present
in the annotated dataset, and with $Y_i^* \in \set{0, 1}$ whether it should be present in 
the ground-truth. In the work of \cite{Jain16}, the authors proposed to take into account the missing labels by 
replacing stochastic estimates of the form $f^*(Y)$ by unbiased estimates $g$ s.t. 
$\expect[g(Y)] = \expect[f^*(Y^*)]$. They derived expressions for cases in which 
$y = 0 \implies f^*(y) = 0$ (e.g. P@k), as well as for the hamming loss. In the following,
we extend these results to all loss functions that decompose over labels.

\subsection{Generic Results}
We first derive an unbiased estimator for an arbitrary loss in the case of single-label 
binary classification \autoref{thm:single_label}, and then extend this to all multi-label 
loss functions that separate over labels in Corollary \ref{cor:multi-label}. 
\begin{theorem}
	\label{thm:single_label}
	Let $(\Omega, \mathcal{F}, P)$ be a probability space with RVs
	$\defmap{Y}{\Omega}{\{0, 1\}}$ and $\defmap{Y^{*}}{\Omega}{\{0, 1\}}$. 
	We assume a propensity model as above, i.e. 
	\begin{align}
		\prob{Y=1, Y^{*}=0} &= 0. \label{cond:nofp} \\
		\prob{Y=1 | Y^*=1}  &\eqqcolon p. \label{eq:def_prop}
	\end{align}
	Let $\defmap{l^*}{\{0,1\} \times \real}{\real}$ be a function which can be
	decomposed into
	\begin{align}
		l^*(y, \hat{y}) \eqqcolon 
		\begin{cases}
			\opl(\hat{y}) & y = 1 \\
			\oml(\hat{y}) & y = 0
		\end{cases}. \label{eq:loss-decomp}
	\end{align}
	Then the function $\defmap{l}{\{0,1\} \times \real}{\real}$ defined as
	\begin{align}
	    l_{+}(\hat{y}) &\coloneqq p^{-1} \left(\opl(\hat{y}) + (p-1) \oml(\hat{y}) \right) \label{eq:generic-f+} \\
		l(y, \hat{y}) &= 
		\begin{cases}
			l_{+}(\hat{y}) & y = 1 \\
			\oml(\hat{y}) & y = 0
		\end{cases}
	\end{align}
	allows to calculate an unbiased estimate of $l^*$:
	\begin{align}
		\expect[l^*(Y^*, \hat{y})] = \expect[l(Y, \hat{y})].
	\end{align}
\end{theorem}
\begin{proof}
	Combining the law of total probability with \eqref{cond:nofp} and \eqref{eq:def_prop}
	gives
	\begin{align}
			\prob{Y=1} 
			\detail{&= \detailcolor \prob{Y=1|Y^*=0} \prob{Y^*=0}  \nonumber \\ & \quad + \prob{Y=1|Y^*=1} \prob{Y^*=1}\\} 
			\detail{&= \detailcolor 0 \cdot \prob{Y^*=0} + p \prob{Y^*=1}\\}
			&= p \cdot \prob{Y^*=1}.
	\end{align}
	The decomposition \eqref{eq:loss-decomp} of $l^*$ can be rewritten as
	\begin{align}
		l^*(y, \hat{y}) = \ind\{y=1\} \left(\opl(\hat{y}) - \oml(\hat{y})\right) + \oml(\hat{y})
	\end{align}

	Let $P$ and $P^*$ denote the probability mass functions of $Y$ and $Y^*$. Importance sampling \cite{press2007numerical} lets us change the random variable by introducing the corresponding
	importance factors
	\detail{{
	\detailcolor
	\begin{align}
		\expect[f(Y^*)]
		&= \sum_{y} \left[ P^*(y) f(y) \right] \\
		&= \sum_{y} \left[ P(y) \frac{P^*(y)}{P(y)} f(y) \right] \\
		&= \expect\left[f(Y) \frac{P^*(Y)}{P(Y)}\right]. \label{eq:imp-samp-def}
	\end{align}
	Applied to $l^*$ this results in
	}
	}
	\begin{align}
		\MoveEqLeft \expect[l^*(Y^*, \hat{y})] 
		\detail{= \detailcolor \expect\left[\ind\{Y^*=1\} \left(\opl(\hat{y}) - \oml(\hat{y})\right) + \oml(\hat{y}) \right] \nonumber \\&}
		= \expect\left[\ind\{Y^{*}=1\} \right] \left(\opl(\hat{y}) - \oml(\hat{y})\right) + \oml(\hat{y}) \nonumber \\
		&= \expect\left[\ind\{Y=1\} \frac{P^*(Y)}{P(Y)} \right] \left(\opl(\hat{y}) - \oml(\hat{y}) \right) + \oml(\hat{y}) \nonumber \\
		\detail{&= \detailcolor \expect\left[\ind\{Y=1\} \frac{P^*(1)}{P(1)} \right] \left(\opl(\hat{y}) - \oml(\hat{y}) \right) + \oml(\hat{y}) \nonumber \\}
		&= \expect\left[\ind\{Y=1\} \frac{\opl(\hat{y}) - \oml(\hat{y})}{p} + \oml(\hat{y}) \right],
	\end{align}
	where we have used the indicator function to replace $P^*(Y)/P(Y) = P^*(1)/P(1) = p^{-1}$. 
	Splitting 
	\begin{align}
	\oml(\hat{y}) = (\ind\{Y=1\} + \ind\{Y=0\}) \oml(\hat{y}),
	\end{align}
	we get the desired statement
	\begin{align}
	    \detail{\MoveEqLeft}
		\expect[l^*(Y^*, \hat{y})] 
		\detail{= \detailcolor \expect\left[\ind\{Y=1\} \frac{\opl(\hat{y}) - \oml(\hat{y})}{p} \right. \nonumber \\ & \qquad + \left. (\ind\{Y=1\} + \ind\{Y=0\}) \oml(\hat{y}) \right] \\}
		\detail{&= \detailcolor \expect\left[\ind\{Y=1\} \frac{\opl(\hat{y}) - \oml(\hat{y}) + p\oml(\hat{y})}{p} \right. \nonumber \\ &\qquad + \left. \ind\{Y=0\} \oml(\hat{y}) \right] \\}
		&= \expect\left[\ind\{Y=1\}l_{+}(\hat{y}) + \ind\{Y=0\}\oml(\hat{y}) \right].
	\end{align}
\end{proof}

\begin{corollary}
    \label{cor:multi-label}
	In an extreme classification setting with $N$ labels, where for each label the conditions of
	\autoref{thm:single_label} are fulfilled, any loss function that decomposes over labels $\mathbf{Y} = \set{Y_1, \ldots, Y_N}$ 
	can	be estimated without bias. Let $\defmap{L}{\set{0,1}^N \times \mathds{R}^N}{\mathds{R}}$
	such that
	\begin{align}
		L^*(\mathbf{y}, \mathbf{\hat{y}}) &= \sum_{i=1}^N l^*_i(y_i, \hat{y}_i),
	\intertext{
	and let $l_i$ be related to $l_i^*$ as prescribed by \autoref{thm:single_label}, equation \eqref{eq:generic-f+}. Defining 
	}
		L(\mathbf{y}, \mathbf{\hat{y}}) &\coloneqq \sum_{i=1}^N l_i(y_i, \hat{y}_i),
	\end{align}
	then $L(\mathbf{Y}, \mathbf{\hat{y}})$ is an unbiased estimate of $L^*(\mathbf{Y^*}, \mathbf{\hat{y}})$.
\end{corollary}
\begin{proof}
	The proof follows by linearity of expectation.
	\detail{
	{
	\detailcolor
	\begin{align}
	\MoveEqLeft
		\expect[L^*(\mathbf{Y^*}, \mathbf{\hat{y}})] = \expect\left[ \sum_{i=1}^N l^*_i(Y^*_i, \hat{y}_i) \right] 
		\\ &= \sum_{i=1}^N \expect\left[ l^*_i(Y^*_i, \hat{y}_i) \right] 
		= \sum_{i=1}^N \expect\left[ l_i(Y_i, \hat{y}_i) \right] \\
		&= \expect \left[\sum_{i=1}^N  l_i(Y_i, \hat{y}_i) \right] = \expect[L(\mathbf{Y}, \mathbf{\hat{y}})]
	\end{align}
	}
	}
\end{proof}

\subsection{Examples of Loss Functions}
We can now apply \autoref{thm:single_label} to the case of different loss functions.
Due to the Corollary above, it suffices to show the following results for the case 
of a single binary label. Note that for binary labels with predictions in $\hat{y} \in \{0, 1\}$,
squared error, 0-1-loss and Hamming loss are identical.
\paragraph{Squared Error}
The decomposition of 
\begin{align}
L^*(y, \hat{y}) &= (y - \hat{y})^2
\shortintertext{is}
l^*_+(\hat{y}) &= 1 - 2\hat{y} + \hat{y}^2\\
l^*_-(\hat{y}) &= \hat{y}^2,
\end{align}
resulting in
\begin{align}
    l_+(\hat{y}) = \frac{1 - 2\hat{y} + \hat{y}^2 + (p-1)\hat{y}^2}{p} = \frac{1 - 2\hat{y}}{p} + \hat{y}^2 
\end{align}
Combining with $l^*_-$ recovers the result of \cite{Jain16}
\begin{align}
    L(y, \hat{y}) = y \frac{1 - 2\hat{y}}{p} + y \hat{y}^2 + (1-y) \hat{y}^2
    = y \frac{1 - 2\hat{y}}{p} + \hat{y}^2.
\end{align}

\paragraph{Hinge Loss}
This loss is a convex upper-bound on the 0-1 loss, typically defined over $\{-1, +1\}$. 
Therefore, we introduce the notation $z=2y-1$, $\hat{z} = 2\hat{y}-1$, 
such that the hinge loss
is
\begin{align}
	L^*(z, \hat{z}) = \ppart{1 - z \hat{z}}.
\end{align}
This corresponds to 
\begin{align}
    \opl(\hat{z}) &= \ppart{1 - \hat{z}}, \\
    \oml(\hat{z}) &= \ppart{1 + \hat{z}},
\end{align}
and thus 
\begin{align}
	l_+(\hat{z}) = p^{-1} \left(\ppart{1 - \hat{z}} + (p-1) \ppart{1 + \hat{z}} \right).
	\label{eq:non-cvx-hinge}
\end{align}
Therefore, using 
\begin{align}
\ind\{y=1\} &= (z + 1)/2,\\ 
\ind\{y=0\} &= (1-z)/2,
\end{align}
the re-weighted loss becomes
\begin{align}
	L(z, \hat{z}) =& \frac{z+1}{2} \frac{\ppart{1 - \hat{z}} + (p-1)\ppart{1 + \hat{z}}}{p} \nonumber \\ 
	& + \frac{1-z}{2} \ppart{1 + \hat{z}}.
\end{align}
This is not a convex function (see \autoref{fig:hinge-losses-sketch}), and as such might be undesirable from an optimization perspective.

Thus we provide a different view, which instead of calculating an unbiased estimate of the convex upper-bound, calculates
a convex upper-bound of an unbiased estimate of the 0-1 loss. 
For the 0-1 loss, we have
\begin{align}
l^*_+(\hat{z}) &=
\begin{cases}
	1 &  \hat{z} < 0 \\
	0 &  \hat{z} \ge 0
\end{cases}
&
	l^*_-(\hat{z}) &= \begin{cases}
	0 &  \hat{z} < 0 \\
	1 &  \hat{z} \ge 0
\end{cases}
\end{align}
This results in the correction
\begin{align}
l_+(\hat{z}) &=
\begin{cases}
	1/p     &  \hat{z} < 0 \\
	1 - 1/p & \hat{z} \ge 0
\end{cases} \label{eq:rw-01-loss-noshift}
\end{align}
This will give an unbiased estimate of the 0-1-loss in the case of missing
labels, but for interpretability it has the disadvantage that it can be
negative. 
Therefore, we use the following equivalent (for
optimization) loss function
\begin{align}
	\tilde{l}_+(\hat{z}) &=
\begin{cases}
	2/p - 1 &  \hat{z} < 0 \\
	0 & \hat{z} \ge 0
\end{cases}, \label{eq:rw-0-1-loss}
\end{align}
which is just shifted by a constant.

A convex upper-bound of this is given by 
\begin{align}
\tilde{l}_+^{\mathrm{cv}}(\hat{z}) &= \ppart{(2/p - 1)(1-\hat{z})} \\
&= (2/p - 1) \ppart{1-\hat{z}}
\end{align}
Combining with the unchanged negative part we get
\begin{align}
	L_{\mathrm{cv}}(y, \hat{y}) &= \left[ \frac{z+1}{2} \frac{2-p}{p} + \frac{1-z}{2} \right] \ppart{1- z \hat{z}} \nonumber \\
	 &= \frac{z(1-p) + 1}{p} \ppart{1 - z \hat{z}} \label{eq:rw-hinge-loss}
\end{align}

\paragraph{Squared Hinge Loss}
To get the squared hinge loss, instead of bounding \eqref{eq:rw-0-1-loss} piecewise-linearly, we
derive a quadratic upper bound. The cusp of the parabola should be at $\hat{z}=1$,
thus we parameterize as $l_+(\hat{z}) = \alpha\ppart{\hat{z}-1}^2$, 
which leads to $\alpha = 2/p - 1$. Thus the loss function becomes
\begin{align}
	L_{\mathrm{cv}}(y, \hat{y}) &= \left[ \frac{z+1}{2} \frac{2-p}{p} + \frac{1-z}{2} \right] \ppart{1 - z \hat{z}}^2 \nonumber \\
	 &= \frac{z(1-p) + 1}{p} \ppart{1 - z \hat{z}}^2 \label{eq:rw-hinge-sq}
\end{align}
Note that this is a different result from squaring the reweighed hinge loss \eqref{eq:rw-hinge-loss}, which results in a quadratic dependency on the propensity.

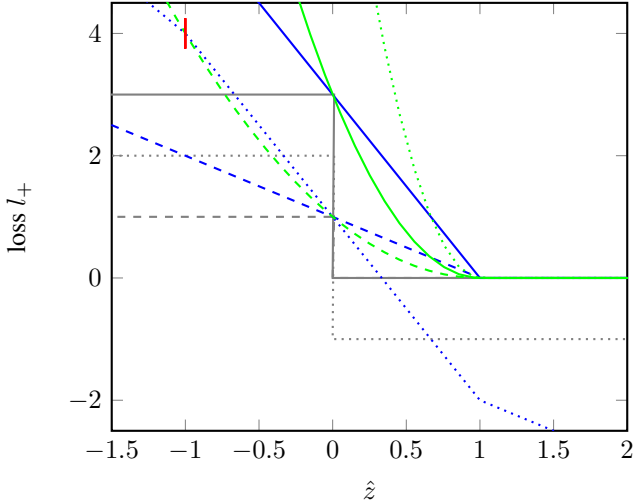
\begin{figure}[t]
    \centering
    \input{loss-sketch}
    \vspace{-2em} 
    \caption{Visualization of 0-1-loss (gray), hinge loss (blue) and squared hinge loss (green) for $y=1$. 
    The dashed lines show the original loss functions without propensity re-weighting. 
    The dotted gray line indicates the re-weighted 0-1-loss without the constant shift \eqref{eq:rw-01-loss-noshift}, the dotted blue line the unbiased, non-convex re-weighted hinge loss
    \eqref{eq:non-cvx-hinge} (the red circle indicates the kink where the non-convexity appears) the dotted green line the square of the convex, re-weighted hinge loss.
    The solid lines indicate \eqref{eq:rw-0-1-loss}, \eqref{eq:rw-hinge-loss} and \eqref{eq:rw-hinge-sq} respectively. 
    }
    \label{fig:hinge-losses-sketch}
\end{figure}

\paragraph{Binary Cross-Entropy}
The BCE loss is given by 
\detaildisplay{l^*_+(\hat{y}) = -\log \hat{y} \detail{\quad} \text{ and } \detail{\quad}  l^*_-(\hat{y}) = -\log( 1-\hat{y}).} 
Plugging into \eqref{eq:generic-f+} gives
\begin{align}
    l_{+}(\hat{y}) = p^{-1} \left(-\log \hat{y} + (1-p) \log (1-\hat{y}) \right)
\end{align}
Combining gives
\begin{align}
    L(y, \hat{y}) &= -\frac{y}{p} \log \hat{y} + \frac{y(1-p) -p+ py }{p} \log( 1-\hat{y}) \nonumber \\
     &= -\frac{y}{p} \log \hat{y} - \left(1 -\frac{y}{p} \right) \log( 1-\hat{y})
\end{align}
This result also follows directly from the fact that the BCE loss is linear in $y$.
\detail{
	{
	\detailcolor
	\begin{align}
	    \MoveEqLeft
		\expect[L^*(Y^*, \hat{y}] = \expect\left[ -Y^* \log \hat{y} - (1 - Y^*) \log( 1-\hat{y}) \right] \nonumber \\
		&= -\expect[Y^*] \log \hat{y} + (1 - \expect[Y^*]) \log( 1-\hat{y}) \\
		&= -\expect[Y/p] \log \hat{y} + (1 - \expect[Y/p]) \log( 1-\hat{y}) \\
		&= -\expect[Y/p] \log \hat{y} + (\expect[Y/p]-1) \log( 1-\hat{y}) \\
		&= \expect\left[- Y/p \log \hat{y} - (1 - Y/p) \log( 1-\hat{y}) \right]
	\end{align}
	}
}

To show that the re-weighted loss functions provide training benefits, we
tested them in a controlled environment. We took the MNIST dataset,
interpreted as a multi-label classification problem, and removed a percentage
of the labels. Compared to the XMC settings presented in \autoref{sec:exps},
this has the advantage that one does not have to use an empirical model to
estimate the propensities. Furthermore, data imbalance can be another
motivator for re-weighting the loss function, and is a prominent feature of
long-tailed extreme classification data sets. The results, in terms of loss
function and classification accuracy, are presented in
\autoref{fig:loss_acc_prop}.

\input{mnist}

%% file: loss-sketch.tex
\begin{tikzpicture}
	\begin{axis}[
	xlabel={$\hat{z}$},
	ylabel={loss $l_+$},
	xmin=-1.5, xmax=2, ymax=4.5, ymin=-2.5,
	thick
	]
		\addplot[gray,dashed] coordinates{
			(-2, 1)
			(0.01, 1)
			(0, 0)
			(2, 0)
		};
		
		\addplot[gray, dotted] coordinates{
			(-2, 2)
			(0.01, 2)
			(0, -1)
			(2, -1)
		};
		\addplot[gray] coordinates{
			(-2, 3)
			(0.01, 3)
			(0, 0)
			(2, 0)
		};

		\addplot[blue,dashed] coordinates{
			(-2, 3)
			(1, 0)
			(2, 0)
		};

		\addplot[blue] coordinates{
			(-1, 6)
			(1, 0)
			(2, 0)
		};

		\addplot[blue,dotted] coordinates{
			(-2, 6)
			(-1, 4)
			(0, 1)
			(1, -2)
			(2, -3)
		};

		\addplot[green,dashed,samples=100] {(x<=1.0) * ((x-1.0)*(x-1.0)};
		\addplot[green,samples=100] {(x<=1.0) * (3*(x-1.0)*(x-1.0)};
		\addplot[green,dotted,samples=100] {(x<=1.0) * (9*(x-1.0)*(x-1.0)};
		
		\draw[red] (axis cs:-1,4) circle[x radius=0.25, y radius=0.25];
	\end{axis}
\end{tikzpicture}

%% file: mnist.tex
\begin{figure*}[tb]
\begin{tikzpicture}
\begin{groupplot}[group style={group size= 2 by 2,
                ylabels at=edge left,
                yticklabels at=edge left,
				xlabels at=edge bottom,
                xticklabels at=edge bottom,},xmin=0, xmax=1, xlabel=propensity, height=5cm, width=0.5\linewidth]
\nextgroupplot[ylabel=loss,ymin=0, ymax=0.25]
\addplot table[x=propensity,y=BCE] {bce.txt};
\addplot table[x=propensity,y=RWBCE] {bce.txt};
\addplot table[x=propensity,y=NBCE] {bce.txt};
\legend{standard, re-weighted, naive}
\nextgroupplot[ymin=0, ymax=0.25]
\addplot table[x=propensity,y=HINGE] {hinge.txt};
\addplot table[x=propensity,y=RWHINGE] {hinge.txt};
\addplot table[x=propensity,y=NCRWHINGE] {hinge.txt};
\legend{standard, re-weighted, non-convex}
\nextgroupplot[ylabel=accuracy, ymin=0.89, ymax=1]
\addplot table[x=propensity,y=BCE] {bce_binary_accuracy.txt};
\addplot table[x=propensity,y=RWBCE] {bce_binary_accuracy.txt};
\addplot table[x=propensity,y=NBCE] {bce_binary_accuracy.txt};
\nextgroupplot[ymin=0.89, ymax=1]
\addplot table[x=propensity,y=HINGE] {hinge_binary_accuracy.txt};
\addplot table[x=propensity,y=RWHINGE] {hinge_binary_accuracy.txt};
\addplot table[x=propensity,y=NCRWHINGE] {hinge_binary_accuracy.txt};
\end{groupplot}
\end{tikzpicture}
\caption{Comparison of loss functions for training MNIST with missing labels. In this case
the MNIST data is interpreted as a multi-label learning problem, and true labels are dropped
with a pre-selected propensity. A single-hidden-layer relu network is trained using different
loss functions for 10 epochs. The resulting classifier is evaluated on MNIST test-data where 
all labels are present, based on test loss and binary classification accuracy. The plots to
the left show different variations of binary cross-entropy with sigmoid activation, to the right 
the hinge loss.
}
\label{fig:loss_acc_prop}
\end{figure*}
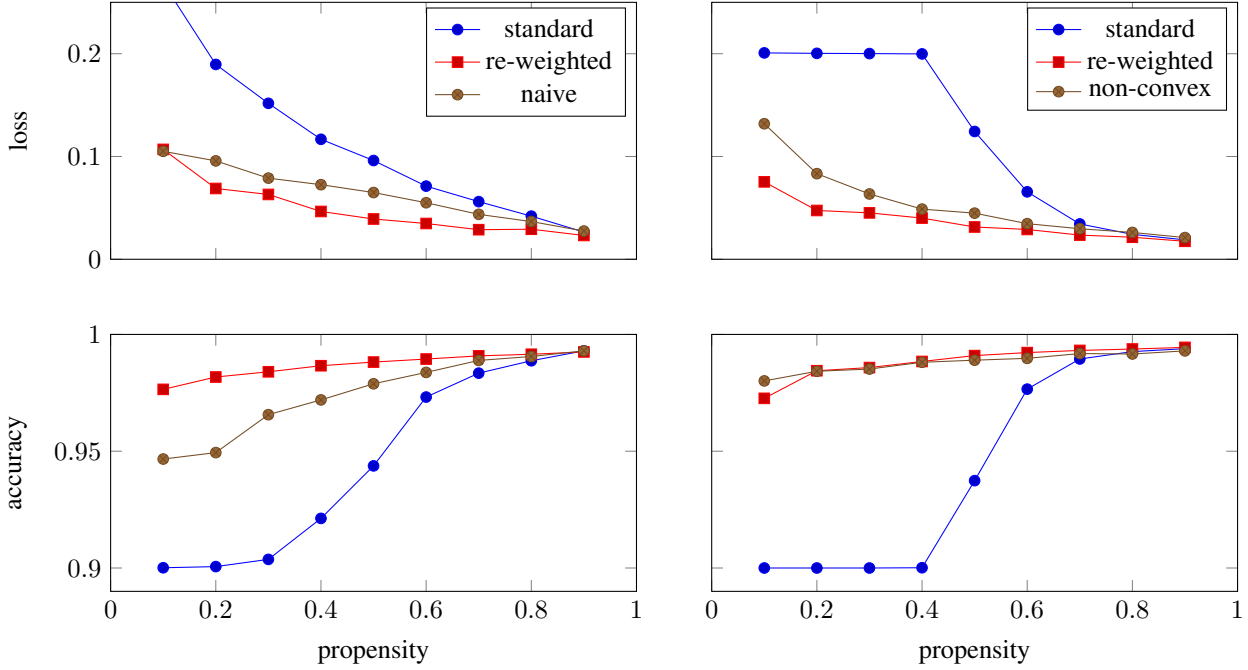

%% file: experiments.tex
\section{Experimental Analysis}\label{sec:exps}

We evaluate the propensity weighting developed in the previous section on publicly available datasets from the extreme classification repository \citep{Bhatia16}. We use the empirical model of \citep[Section 5]{Jain16} for evaluating the label propensities
\begin{align}
p_{\ell} = (1 + C \exp(-A \log (N_l + B)))^{-1},
\end{align}
which is standard in the community. Here $A$, $B$ and \linebreak $C=(\log N - 1)(B+1)^A$ are dataset dependent parameters, and $N_l$ denotes the number of positives for label $l$. The data is represented as (sparse) bag-of-words features which is suitable for scaling with large-scale linear classification algorithms. Also, these datasets exhibit a long-tailed label distribution (\autoref{fig:powerlaw} for WikiLSHTC-325K dataset). The statistics of the datasets as well as A and B for each dataset are presented in Table~\ref{tbl:datasets}

\begin{table*}
    \centering
	\begin{tabular}{l|r|r|r|r|S[table-format=3.2]|S[table-format=3.2]|S[table-format=3.2]|S[table-format=3.2]}
		\toprule
		\textbf{Dataset} &  \# \textbf{Training } &  \# \textbf{Test } & \# \textbf{Labels} &  \# \textbf{Features}& \textbf{APpL} & \textbf{ALpP} & \textbf{A} & \textbf{B}\\ \midrule
		\textbf{EURLex-4K} &  15,539  & 3,809  & \textbf{3993} & 5,000 & 25.7 & 5.3 & 0.55 & 1.5 \\
\textbf{AmazonCat-13K} &  1,186,239  & 306,782  & \textbf{13,330} & 203,882 & 448.5 & 5.04 & 0.55 & 1.5\\
\textbf{Wikipedia-31K} &  14,146  & 6,616  & \textbf{30,938} & 101,938 & 8.5 & 18.6 & 0.55 & 1.5\\
\textbf{WikiLSHTC-325K} &  1,778,351 & 587,084  & \textbf{325,056} & 1,617,899 & 17.4 & 3.2 & 0.5 & 0.4\\
\textbf{Wikipedia-500K} & 1,813,391 & 783,743  &  \textbf{501,070} &  2,381,304 & 24.7 & 4.7 & 0.5 & 0.4\\
\textbf{Amazon-670K} & 490,499 & 153,025 & \textbf{670,091}  & 135,909 & 3.9 & 5.4 & 0.6 & 2.6\\
		\bottomrule
	\end{tabular}
	\caption{The statistics of the multi-label datasets used in our experiments. APpL denotes the average points per label and ALpP is the average labels per point respectively. $A$ and $B$ refer to the parameters of the propensity model.}
\label{tbl:datasets}
\end{table*}

We apply the weighting scheme as part of a one-vs-rest implementation of linear SVM in multi-label setting, which is also referred to binary relevance in the literature \citep{luaces2012binary}. The weight vector $w_{\ell}$ for each label $\ell$ is learnt by minimizing a combination of squared hinge loss and $l_2$-regularization. Separating the squared hinge loss into the contributions from false negatives and false positives for label $\ell$, this is given by the following optimization problem:
\begin{align}\nonumber
  \min_{w_{\ell}}  \|w_{\ell}\|^2_2 
  & + C_{p_{\ell}}\underbrace{\sum_{\mathclap{i\in \mathcal{L}^+_{\ell}}} \max(0,1-(w_{\ell}^T x_i+b_{\ell}))^2}_{\text{FN}} \\
  & + C_{n_{\ell}}\underbrace{\sum_{\mathclap{i\in \mathcal{L}^-_{\ell}}} \max(0,1+(w_{\ell}^T x_i+b_{\ell}))^2}_{\text{FP}} 
\end{align} \label{pwdismec}
where $\mathcal{L}^+_{\ell}$ ($\mathcal{L}^-_{\ell}$) denotes the set of positive (negative) training samples corresponding to the label ${\ell}$.

The objective function for each label ${\ell}$ consists of two parts, FN (false negatives) and FP (false positives). The FN part penalizes samples that have not been assigned to a relevant label, whereas FP penalizes those that have been mistakenly assigned to an irrelevant label.

According to \eqref{eq:rw-hinge-sq}, we get a convex upper-bound on the 0-1 loss when $C_{p_{\ell}} = \left(\frac{2}{p} -1\right)$ and $C_{n_{\ell}} = 1$. However, minimizing the 0-1 loss may not be the best goal in the XMC setting, due to the presence of tail labels. For those, most samples are negative, and thus most contribution to the loss will be from negatives, and recall might be bad for tail labels. To counteract this, we want to weight FN more strongly for tail labels, compared to head labels. Empirically, we found a weighting of $C_{p_{\ell}} = \left(\frac{1}{p} -1\right)$ to work well.

With the above weighting factors, the optimization in equation \eqref{pwdismec} generalizes the objective function optimized by one-vs-rest based extreme classification algorithm \texttt{DiSMEC} \citep{dismec}, where both the weights are set to 1. We thus call our approach \textbf{P}ropensity \textbf{W}eighted DiSMEC (\texttt{PW-DiSMEC}).

\subsection{Evaluation metrics} 
For evaluation of the methods, we use precision at $k$ ($\p{k}$) and normalized discounted cumulative gain at $k$ ($\ndcgatk$), which are standard and commonly used metric for evaluating models in XMC. For each test sample with observed ground truth label vector $\mathbf{y} \in \set{0,1}^\numlabels$ and predicted vector $\mathbf{\hat{y}} \in \real^\numlabels$, $\p{k}$ and $\ndcgatk$ are computed as
{
\setlength{\abovedisplayskip}{0.5\abovedisplayskip}
\setlength{\belowdisplayskip}{0.5\belowdisplayskip}
\begin{align}
\p{k}(\textbf{y},\yhatv) \coloneqq &\frac{1}{k} \sum_{{\ell} \in \rank_k(\yhatv)} \textbf{y}_{\ell}, \label{eq:patk}\\
\ndcgatk(\textbf{y},\yhatv) \coloneqq & \frac{\dcg{k}}{\sum_{l=1}^{\min(k, \|\textbf{y}\|_0)}{\frac{1}{\log(l+1)}}}, \label{eq:ndcg}\\
\dcg{k}(\textbf{y},\yhatv) \coloneqq& \, \sum_{\mathclap{\quad{\ell} \in \rank_k(\yhatv)}} \left({\log(l+1)}\right)^{-1} \mathbf{y}_l,
\end{align}
where $\rank_k(\mathbf{\hat{y}})$ returns the $k$ largest indices of $\mathbf{\hat{y}}$.
}


 With applications of XMC arising in recommendation systems and web-advertising, the objective of an algorithm in this domain is to correctly recommend/advertise among the top-k slots. Hence evaluation metrics such as precision@$k$ and its rank-sensitive version \ndcgatk are commonly used to benchmark extreme classification algorithms. A collection of results from recent papers on datasets in \autoref{tbl:datasets} for algorithms developed over the last few years is given on the extreme classification repository \citep{Bhatia16}.
 
For the case of missing labels, and imbalanced classification arising from the presence of tail-labels, propensity scored variants of $\p{k}$ and \ndcgatk are of main interest to us in this work and are given by :
 \begin{align}
 \psp{k}(\textbf{y},\yhatv)  &\coloneqq  \frac{1}{k} \sum_{\ell \in \rank_k{(\yhatv)}} \frac{\yv_\ell}{p_\ell} \label{eq:propp} \\
 \mathrm{PSnDCG}@k(\textbf{y},\yhatv) & \coloneqq \frac{\mathrm{PSDCG}@k}{\sum_{\ell=1}^{\min(k, \|\yv\|_0)}{\frac{1}{\log(\ell+1)}}}
 \label{eq:prop} \\
 \mathrm{PSDCG}@k(\textbf{y},\yhatv) &\coloneqq \sum_{\ell \in \rank_k{\!(\mathbf{\hat{y}})}}{\frac{\mathbf{y}_\ell}{p_\ell\log(\ell+1)}}.
 \end{align}

 To match against the best possible performance attainable by any system, as suggested in \citep{Jain16}, we define,
  for $M$ test samples, $\mathbb{G}(\{\yhatv\}) = \frac{-1}{M}\sum_{i=1}^{M}\mathbb{L}(\yhatv_i,\textbf{y}_i)$, where $\mathbb{L}(\cdot, \cdot)$ and $\mathbb{G}(\cdot)$ signify loss and gain respectively. We use 
  $100*\mathbb{G}(\{\yhatv\})/\mathbb{G}(\{\textbf{y}\})$ as the performance metric. 
 The loss $\mathbb{L}(\cdot, \cdot)$ can take two forms, (i) $\mathbb{L}(\yhatv_i,\textbf{y}_i) =  - \mathrm{PSnDCG}@k$, and (ii) $\mathbb{L}(\yhatv,\textbf{y}) =  - \psp{k}$.
 This leads to the two metrics which are finally used in our comparison in \autoref{tbl:psp} (denoted $\psp{k}$ and $\mathrm{PSnDCG}@k$), and evaluated for $k=1,3,5$.

\input{results-table}
\input{vanilla_plots}
\subsection{Baseline methods}
We compare the proposed method against seven state-of-the-art algorithms from three classes of XMC methods:
\setlist{nolistsep}
\begin{enumerate}
    \item Label-embedding methods:
    \begin{itemize}
        \item \texttt{LEML} \citep{yu2014large} - Uses a regularized least squares objective for learning to embed the label vectors to a lower-dimensional linear subspace, map the feature vectors to the lower dimensional label space, and decompress the lower dimensional labels by a linear transformation.
        \item \texttt{SLEEC} \citep{bhatia2015sparse} - 
        Embeds the label vectors to a lower-dimensional space in which nearest neighbors are preserved, learns a regressor to map the feature vectors to the embedded space, and performs kNN over the regressor for decompression.
    \end{itemize}
    
    \item Tree-based methods:
    \begin{itemize}
        \item \texttt{PFastXML} \citep{Jain16} - The instances are recursively partitioned into two nodes, with the objective of maximizing a propensity scored metric in each node.
        \item \texttt{Parabel} \citep{Prabhu18b} - Partitions the labels into two balanced groups using 2-means and learns a one-vs-rest classifier at each node.
    \end{itemize}
    
    \item One-vs-all methods:
    \begin{itemize}
        \item \texttt{(P)PD-Sparse} \citep{yenpd, yen2017ppdsparse} - These class of extreme classification algorithms exploit the sparsity in the primal and dual problem combined with elastic net regularization. \texttt{PD-Sparse} uses multi-class hinge loss while \texttt{PPD-Sparse} uses hinge loss for binary classification. Being amenable to distributed training, \texttt{PPD-Sparse} can scale to bigger datasets.
        \item \texttt{DiSMEC} \citep{dismec} - Optimizes Hamming loss with $\ell_2$ regularization in a distributed environment which achieves state-of-the-art on vanilla $\p{k}$. DiSMEC prunes weights for model size reduction.
        \item \texttt{ProXML} \citep{babbar2019data} - Improves tail-label detection by posing the learning problem as an instance of robustness optimization. It proposes to guard against small perturbations in the feature composition of the instances of the same class, leading to $\ell_1$ regularization. Its key difference compared to our method is that it does not weigh the hinge loss but treats the regularization part.
    \end{itemize}
    
\end{enumerate}

\subsection{Experimental Results}

\paragraph{Performance on propensity-scored metrics} A comparison of the proposed method with the baselines in terms of the propensity scored precision is presented in \autoref{tbl:psp}. The proposed method outperforms the embedding-based methods \texttt{LEML} and \texttt{SLEEC} on all benchmarks.
 
In comparison to tree-based methods, our method achieves better results on five out of six datasets. For WikiLSHTC-325K and Wikipedia-500K, the improvements in $\psp{k}$ over \texttt{PfastXML}, in which the objective is to maximize a propensity scored metric, are about 10\% in absolute terms. Over \texttt{Parabel}, the relative improvement is almost 20-30\% on most datasets.

It is clear that the proposed method outperforms \texttt{PD-Sparse} and vanilla \texttt{DiSMEC} on all the datasets. The proposed method which is based on weighting the squared hinge loss function and keeping the regularization as $l_2$ also outperforms \texttt{ProXML}, which is specifically designed for better tail-label detection. Furthermore, being amenable to second order optimization scheme, it is almost two orders of magnitude faster to train than \texttt{ProXML}.

\paragraph{Performance on vanilla metrics} While the proposed method aims at mitigating the impact of missing label and data-imbalance on the performance on tail labels, it is desirable that this should not come at an expense of accuracy in vanilla metrics, in equations (\ref{eq:patk}) \& (\ref{eq:ndcg}). \autoref{fig:vanilla_plots_fig} illustrates a comparison of the proposed method and the baselines in terms of precision and nDCG. In the interest of space, the results are compared to three other baseline methods. As can be seen, the proposed method works at par or even in better in most cases on these metrics as well.

%% file: results-table.tex
\begin{table*}
\centering
\small
\begin{tabular}{l|c|cc|cc|ccc}
\toprule
\multirow{2}{*}{Dataset} & \multicolumn{1}{c}{Ours} & \multicolumn{2}{|c|}{Embedding~based} & \multicolumn{2}{c|}{Tree~based} &  \multicolumn{3}{c}{Linear one-vs-rest}  \\
& \texttt{PW-DiSMEC} & \texttt{SLEEC} & \texttt{LEML} &  \texttt{PfXML} & \texttt{Prbl}  & \texttt{(P)PD} & \texttt{DiSMEC} & \texttt{ProXML} \\ \midrule
\textbf{EURLex-4K} &&& & & & & &\\
~~\textit{PSP@1} & 44.9 & 35.4 & 24.1 & 39.9 & 37.7 & 38.2 & 41.2 & \textbf{45.2}\\
~~\textit{PSP@3} & \textbf{49.2} & 39.8 & 27.2 & 43.0 & 44.7 & 42.7 & 45.4 & 48.2\\
~~\textit{PSP@5} & \textbf{52.2} & 42.7 & 29.1 & 44.5 & 48.8 & 44.8 & 49.3 & 51.0\\
~~\textit{PSnDCG@3} & \textbf{48.0} & 38.8 & 26.4 & 42.2 & 43.4 & 40.9 & 44.3 & 47.5\\
~~\textit{PSnDCG@5} & \textbf{50.1} & 40.3 & 27.7 & 43.2 & 46.1 & 42.8 & 46.9 & 49.1\\

\textbf{AmazonCat-13K} &&&&& & & &\\
~~\textit{PSP@1} & 67.1 & 46.8 & - & \textbf{68.0} & 50.9 & 49.6 & 51.4 & 52.3\\
~~\textit{PSP@3} & 71.3 & 58.5 & - & \textbf{72.3} & 64.0 & 61.6 & 61.0 &  62.0\\
~~\textit{PSP@5} & 73.2 & 65.0 & - & \textbf{74.6} & 72.1 & 68.2 & 65.9 & 66.2\\
~~\textit{PSnDCG@3} & 70.4 & 55.1 & - & \textbf{71.1} & 35.2 & 58.3 & 65.2 & 66.1\\
~~\textit{PSnDCG@5} & 72.4 & 60.1 & - & \textbf{72.6} & 38.1 & 62.7 & 68.8 & 69.5\\

\textbf{Wikipedia-31K} &&& & & & & &\\
~~\textit{P@1} & \textbf{15.6} & 11.1 & 9.4 & 9.4 & 11.7 & - & 13.4 & 13.6\\
~~\textit{P@3} & \textbf{16.2} & 11.9 & 10.1 & 9.9 & 12.7 & - & 12.9 & 13.0\\
~~\textit{P@5} & \textbf{17.3} & 12.4 & 10.6 & 10.4 & 13.7 & - & 13.6 &  13.2\\
~~\textit{PSnDCG@3} & \textbf{16.0} & 11.7 & 9.9 & 9.7 & 12.5 & - & 13.2 & 13.4\\
~~\textit{PSnDCG@5} & \textbf{16.8} & 12.1 & 10.2 & 10.1 & 13.1 & - & 13.6 & 13.7\\


\textbf{WikiLSHTC-325K} &&& & & & & &\\
~~\textit{PSP@1} & \textbf{35.3} & 20.5 & 3.4 & 25.4 & 28.7 & 28.3 & 29.1 & 34.8\\
~~\textit{PSP@3} & \textbf{38.4} & 23.3 & 3.7 & 26.8 & 35.0 & 33.5 & 35.6 & 37.7\\
~~\textit{PSP@5} & \textbf{41.7} & 25.2 & 4.2 & 28.3 & 38.6 & 36.6 & 39.4 & 41.0\\
~~\textit{PSnDCG@3} & \textbf{39.8} & 22.4 & 3.6 & 26.4 & 35.2 & 31.9 & 35.9 & 38.7\\
~~\textit{PSnDCG@5} & \textbf{42.6} & 23.5 & 3.9 & 27.2 & 38.1 & 33.6 & 39.4 & 41.5\\ 

\textbf{Wikipedia-500K} &&& & & & & &\\
~~\textit{PSP@1} & \textbf{33.9} & 21.1 & 3.2 & 22.2 & 28.8 & - & 31.2 & 33.1\\
~~\textit{PSP@3} & \textbf{36.7} & 21.0 & 3.4 & 21.3 & 31.9 & - & 33.4 & 35.0\\
~~\textit{PSP@5} & \textbf{39.9} & 20.8 & 3.5 & 21.6 & 34.6 & - & 37.0 & 39.4\\
~~\textit{PSnDCG@3} & \textbf{39.7} & 20.9 & 3.1 & 21.6 & 31.2 & - & 33.7 & 35.2\\
~~\textit{PSnDCG@5} & \textbf{42.7} & 23.1 & 3.3 & 21.8 & 35.5 & - & 37.1 & 39.0\\ 

\textbf{Amazon-670K} &&& & & & & &\\
~~\textit{PSP@1} & \textbf{30.9} & 20.6 & 2.0 & 27.1 & 27.6 & 26.6 & 27.8 & 30.8\\
~~\textit{PSP@3} & \textbf{33.1} & 23.3 & 2.2 & 28.2 & 31.0 & 30.7 & 30.6 & 32.8\\
~~\textit{PSP@5} & \textbf{35.2} & 26.0 & 2.4 & 29.3 & 34.1 & 34.7 & 34.2 & 35.1\\
~~\textit{PSnDCG@3} & \textbf{31.9} & 22.6 & 2.2 & 27.9 & 28.4 & - & 28.8 & 31.7\\
~~\textit{PSnDCG@5} & \textbf{32.9} & 24.4 & 2.3 & 28.6 & 29.9 & - & 30.7 & 32.6\\

\bottomrule
\end{tabular}

\caption{ Comparison of PSP@k and PSnDCG@k on benchmark datasets for $k=1,3\text{ and }5$. 
The columns \texttt{PfXML}, \texttt{Prbl}, \texttt{(P)PD} show the results for \texttt{PfastXML}, \parabel, and \texttt{(P)PD-Sparse} respectively.
For each row, score of the best performing algorithm is highlighted in bold. 
Entries marked "-" imply the corresponding method could not scale to the particular dataset (LEML for AmazonCat-13K and PD-Sparse for Wikipedia-500K) or the scores are unavailable (PDSparse for Wikipedia-31K). The comparison of the models in terms of vanilla precision and nDCG can be found in \autoref{fig:vanilla_plots_fig}.}
\label{tbl:psp}
\end{table*}

%% file: vanilla_plots.tex
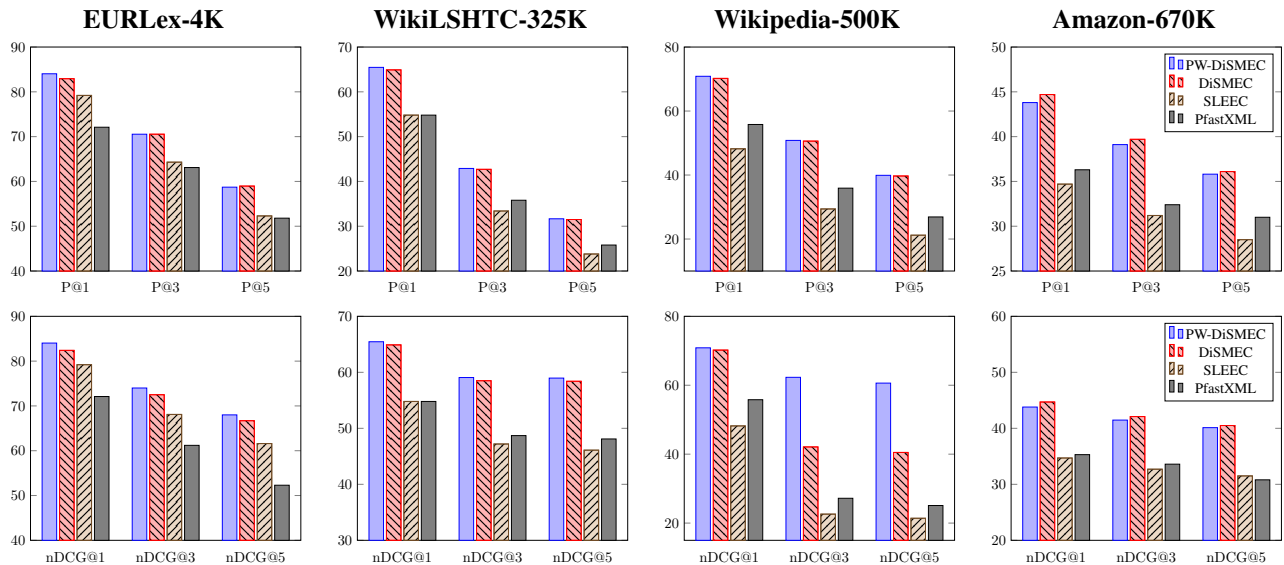
\begin{figure*}[bt]
  \centering
  \begin{tabular}{cccc}
    \textbf{EURLex-4K} & \textbf{WikiLSHTC-325K} & \textbf{Wikipedia-500K} & \textbf{Amazon-670K} \\
    \begin{tikzpicture}[scale=0.55]
      \begin{axis}[
        major x tick style = transparent,
        height=7cm,
        ybar,
        enlarge x limits=0.25,
        symbolic x coords={\p{1},\p{3},\p{5}},
        xtick = data,
        ymin=40, ymax=90, 
        legend style={at={(0.97,0.97)},anchor=north east}]
        \addplot coordinates{ (\p{1}, 84.03) (\p{3}, 70.54) (\p{5}, 58.72)};
        \addplot +[postaction={pattern=north west lines}] coordinates{ (\p{1}, 82.93) (\p{3}, 70.54) (\p{5}, 58.96)};
        \addplot +[postaction={pattern=north east lines}] coordinates{ (\p{1}, 79.2) (\p{3}, 64.3) (\p{5}, 52.3)};
        \addplot coordinates{ (\p{1}, 72.1) (\p{3}, 63.1) (\p{5}, 51.8)};
        
      \end{axis}
    \end{tikzpicture}  &
                         \begin{tikzpicture}[scale=0.55]
                           \begin{axis}[
                             major x tick style = transparent,
                             height=7cm,
                             ybar,
                             enlarge x limits=0.25,
                             symbolic x coords={\p{1},\p{3},\p{5}},
                             xtick = data,
                             ymin=20, ymax=70, 
                             legend style={at={(0.03,0.97)},anchor=north west}]
                             \addplot coordinates{ (\p{1}, 65.46) (\p{3}, 42.90) (\p{5}, 31.66)};
                             \addplot +[postaction={pattern=north west lines}] coordinates{ (\p{1}, 64.9) (\p{3}, 42.7) (\p{5}, 31.5)};
                             \addplot +[postaction={pattern=north east lines}] coordinates{ (\p{1}, 54.8) (\p{3}, 33.4) (\p{5}, 23.8)};
                             \addplot coordinates{ (\p{1}, 54.8) (\p{3}, 35.8) (\p{5}, 25.8)};
                             
                           \end{axis}
                         \end{tikzpicture}&
                                            \begin{tikzpicture}[scale=0.55]
                                              \begin{axis}[
                                                major x tick style = transparent,
                                                height=7cm,
                                                ybar,
                                                enlarge x limits=0.25,
                                                symbolic x coords={\p{1},\p{3},\p{5}},
                                                xtick = data,
                                                ymin=10, ymax=80, 
                                                legend style={at={(0.03,0.97)},anchor=north west}]
                                                \addplot coordinates{ (\p{1}, 70.86) (\p{3}, 50.82) (\p{5}, 39.91)};
                                                \addplot +[postaction={pattern=north west lines}] coordinates{ (\p{1}, 70.2) (\p{3}, 50.6) (\p{5}, 39.7)};
                                                \addplot +[postaction={pattern=north east lines}] coordinates{ (\p{1}, 48.2) (\p{3}, 29.4) (\p{5}, 21.2)};
                                                \addplot coordinates{ (\p{1}, 55.8) (\p{3}, 35.9) (\p{5}, 26.9)};
                                                
                                              \end{axis}
                                            \end{tikzpicture} & 
                                                        
                                                          \begin{tikzpicture}[scale=0.55]
                                        \begin{axis}[
                                                major x tick style = transparent,
                                                height=7cm,
                                                ybar,
                                                enlarge x limits=0.25,
                                                symbolic x coords={\p{1},\p{3},\p{5}},
                                                xtick = data,
                                                ymin=25, ymax=50, 
                                                legend style={at={(0.97,0.97)},anchor=north east}]
                                                \addplot coordinates{ (\p{1}, 43.80) (\p{3}, 39.11) (\p{5}, 35.81)};
                                                \addplot +[postaction={pattern=north west lines}] coordinates{ (\p{1}, 44.7) (\p{3}, 39.7) (\p{5}, 36.1)};
                                                \addplot +[postaction={pattern=north east lines}] coordinates{ (\p{1}, 34.7) (\p{3}, 31.2) (\p{5}, 28.5)};
                                                \addplot coordinates{ (\p{1}, 36.3) (\p{3}, 32.4) (\p{5}, 31.0)};
        
                                            \legend{PW-DiSMEC, DiSMEC, SLEEC, PfastXML}
                                        \end{axis}
                                    \end{tikzpicture}\\

    \begin{tikzpicture}[scale=0.55]
      \begin{axis}[
        major x tick style = transparent,
        height=7cm,
        ybar,
        enlarge x limits=0.25,
        symbolic x coords={\ndcg{1},\ndcg{3},\ndcg{5}},
        xtick = data,
        ymin=40, ymax=90, 
        legend style={at={(0.03,0.97)},anchor=north west}]
        \addplot coordinates{ (\ndcg{1}, 84.03) (\ndcg{3}, 74.00) (\ndcg{5}, 68.00)};
                            \addplot +[postaction={pattern=north west lines}] coordinates{ (\ndcg{1}, 82.4) (\ndcg{3}, 72.5) (\ndcg{5}, 66.7)};
                            \addplot +[postaction={pattern=north east lines}] coordinates{ (\ndcg{1}, 79.2) (\ndcg{3}, 68.1) (\ndcg{5}, 61.6)};
                            \addplot coordinates{ (\ndcg{1}, 72.1) (\ndcg{3}, 61.2) (\ndcg{5}, 52.3)};
        
      \end{axis}
    \end{tikzpicture} &
                        \begin{tikzpicture}[scale=0.55]
                          \begin{axis}[
                            major x tick style = transparent,
                            height=7cm,
                            ybar,
                            enlarge x limits=0.25,
                            symbolic x coords={\ndcg{1},\ndcg{3},\ndcg{5}},
                            xtick = data,
                            ymin=30, ymax=70, 
                            legend style={at={(0.03,0.97)},anchor=north west}]
                            \addplot coordinates{ (\ndcg{1}, 65.46) (\ndcg{3}, 59.07) (\ndcg{5}, 58.97)};
                            \addplot +[postaction={pattern=north west lines}] coordinates{ (\ndcg{1}, 64.9) (\ndcg{3}, 58.5) (\ndcg{5}, 58.4)};
                            \addplot +[postaction={pattern=north east lines}] coordinates{ (\ndcg{1}, 54.8) (\ndcg{3}, 47.2) (\ndcg{5}, 46.1)};
                            \addplot coordinates{ (\ndcg{1}, 54.8) (\ndcg{3}, 48.7) (\ndcg{5}, 48.1)};
                            
                          \end{axis}
                        \end{tikzpicture} & 
                                            \begin{tikzpicture}[scale=0.55]
                                              \begin{axis}[
                                                major x tick style = transparent,
                                                height=7cm,
                                                ybar,
                                                enlarge x limits=0.25,
                                                symbolic x coords={\ndcg{1},\ndcg{3},\ndcg{5}},
                                                xtick = data,
                                                ymin=15, ymax=80, 
                                                legend style={at={(0.03,0.97)},anchor=north west}]
                                                \addplot coordinates{ (\ndcg{1}, 70.86) (\ndcg{3}, 62.30) (\ndcg{5}, 60.62)};
                                                \addplot +[postaction={pattern=north west lines}] coordinates{ (\ndcg{1}, 70.2) (\ndcg{3}, 42.1) (\ndcg{5}, 40.5)};
                                                \addplot +[postaction={pattern=north east lines}] coordinates{ (\ndcg{1}, 48.2) (\ndcg{3}, 22.6) (\ndcg{5}, 21.4)};
                                                \addplot coordinates{ (\ndcg{1}, 55.8) (\ndcg{3}, 27.2) (\ndcg{5}, 25.1)};
                                                
                                              \end{axis}
                                            \end{tikzpicture} & 
                                                \begin{tikzpicture}[scale=0.55]
                                              \begin{axis}[
                                                major x tick style = transparent,
                                                height=7cm,
                                                ybar,
                                                enlarge x limits=0.25,
                                                symbolic x coords={\ndcg{1},\ndcg{3},\ndcg{5}},
                                                xtick = data,
                                                ymin=20, ymax=60, 
                                                legend style={at={(0.97,0.97)},anchor=north east}]
                                                \addplot coordinates{ (\ndcg{1}, 43.80) (\ndcg{3}, 41.47) (\ndcg{5}, 40.12)};
                                                \addplot +[postaction={pattern=north west lines}] coordinates{ (\ndcg{1}, 44.7) (\ndcg{3}, 42.1) (\ndcg{5}, 40.5)};
                                                \addplot +[postaction={pattern=north east lines}] coordinates{ (\ndcg{1}, 34.7) (\ndcg{3}, 32.7) (\ndcg{5}, 31.5)};
                                                \addplot coordinates{ (\ndcg{1}, 35.3) (\ndcg{3}, 33.6) (\ndcg{5}, 30.8)};
                                                
                                                \legend{PW-DiSMEC, DiSMEC, SLEEC, PfastXML}
                                              \end{axis}
                                            \end{tikzpicture}
  \end{tabular}
  \caption{Comparison of P@k (top row) and  nDCG@k (bottom row) on benchmark datasets. The comparison of the models in terms of the propensity scored metrics can be found in \autoref{tbl:psp}.}
  \label{fig:vanilla_plots_fig}
\end{figure*}

%% file: related.tex
\section{Other Related work}
Various works in XMC can be broadly divided into five main categories as follows :
\begin{enumerate}

\item \textit{Tree-based} : Tree-based methods implement a divide-and-conquer paradigm and scale to large label sets in XMC by partitioning the labels space. As a result, these scheme of methods have the computational advantage of enabling faster training and prediction. Apart from those discussed in previous sections, some of the other methods include recent works ~\citep{jasinska2016extreme,majzoubi2019ldsm, wydmuch2018no}. However, tree-based methods suffer from error propagation in the tree cascade, and perform particularly worse on metrics which are sensitive for tail-labels. An approach based on shallower trees which mitigates the effect of error cascading has been recently demonstrated in ~\citep{khandagale2019bonsai}.
Approaches based on decision trees have also been proposed for multi-label classification and those tailored to XMC regime \citep{si2017gradient}.

\item \textit{Label embedding} : Label-embedding approaches assume that, despite large number of labels, the label matrix is effectively low rank and therefore project it to a low-dimensional sub-space.
These approaches have been at the fore-front in multi-label classification for small scale problems with few tens or hundred labels \citep{hsu2009multi, weston2011wsabie, linmulti}. 
In some of the works, it was argued that the low rank embedding may be insufficient for capturing the label diversity in XMC settings (\citep{xurobust, bhatia2015sparse, tagami2017annexml}), which has been questioned in the recent work \citep{guo2019breaking} .

\item \textit{One-vs-rest} : As the name suggests, these methods learn a classifier per label which distinguishes it from rest of the labels. 
In terms of prediction accuracy, these methods have been shown to be among the best performing ones for XMC~\citep{dismec,yen2017ppdsparse, babbar2019data}. However, due to their reliance on a distributed training framework, it remains challenging to employ them in resource constrained environments.

\item \textit{Deep learning} : Deeper architectures on top of word-embeddings have also been explored in recent works. A convolutional network based approach, \texttt{XML-CNN}, for deep extreme multi-label classification was proposed in \citep{liu2017deep}. Recently, successful application of attention mechanism, originally motivated from a machine translation perspective, has been demonstrated in the recent work \texttt{AttentionXML} \citep{you2019attentionxml}. For dense feature vectors, \texttt{AttentionXML} performs the best among deep learning approaches significantly improving over \texttt{XML-CNN}. \texttt{X-Bert}, an approach based on pre-trained Bert language model (\citep{devlin2018bert}) has been presented in the work \citep{chang2019modular}. 

\item \textit{Negative Sampling based methods} : Recently, there has also been a surge of methods in designing efficient negative sampling based methods for extreme classification \citep{reddi2019stochastic, jain2019slice, Bamler2020Extreme}. The primary goal of these algorithms is to avoid computing the loss over all the samples which do not belong a given label, and hence speed up training without any significant loss in prediction accuracy.
\end{enumerate}

%% file: conclusion.tex
\section{Conclusion}
\label{sec:conclusion}
Towards modeling the missing labels and extreme data imbalance when learning with number of labels, we presented an derived and analyzed unbiased loss functions which decompose over the individual labels, which includes the commonly used variants such as hinge- and squared-hinge-loss and binary cross-entropy loss. For the setting of XMC, the suitability of the estimator for squared-hinge-loss function was demonstrated by incorporating in the one-vs-rest implementation of large-scale linear SVM. Empirically, it was observed that the resulting method outperforms all existing baselines on benchmark datasets.